\documentclass[11pt]{article}
\usepackage[utf8]{inputenc}
\usepackage[T1]{fontenc}
\usepackage{lmodern}
\usepackage[a4paper,margin=1in]{geometry}
\usepackage{amsmath,amssymb,amsthm,mathtools}
\usepackage{booktabs,array,longtable,tabularx,multirow}
\usepackage{enumitem}
\usepackage{hyperref}
\usepackage[numbers]{natbib}
\usepackage{microtype}
\usepackage{setspace}
\usepackage{xcolor}
\usepackage{url}
\usepackage{graphicx}
\usepackage{float}
\usepackage{placeins}
\usepackage{csquotes}
\hypersetup{colorlinks=true,linkcolor=blue,citecolor=blue,urlcolor=blue}

\newtheorem{definition}{Definition}
\newtheorem{theorem}{Theorem}
\newtheorem{proposition}{Proposition}
\newtheorem{lemma}{Lemma}

\newtheorem{remark}{Remark}
\newtheorem{assumption}{Assumption}

\newcommand{\RoSO}{\mathrm{RoSO}}
\newcommand{\E}{\mathcal E}
\newcommand{\Hh}{\mathcal H}
\newcommand{\Lcal}{\mathcal L}
\newcommand{\Mcal}{\mathcal M}
\newcommand{\Tcal}{\mathcal T}
\newcommand{\Z}{\mathcal Z}
\newcommand{\G}{\mathcal G}

\newcommand{\Id}{\mathrm{Id}}

\title{From Ontology Conformance to Admissible Reconfiguration: A RoSO–SMGI Adequacy Argument for Robotic Service Governance}
\author{Aomar Osmani}
\date{March 2026}

\begin{document}
\maketitle

\begin{abstract}
The Robotic Service Ontology (RoSO) gives service robotics a typed semantic vocabulary for services, functions, interactions, and deployment-sensitive constraints. Its public revision trail makes visible a harder question than ontology conformance alone can settle: once a service is rebound, recomposed, repaired, or redeployed, under what conditions does the resulting configuration remain an admissible realization of the same protected service? This article argues that the Structural Model of General Intelligence (SMGI) is relevant exactly at that level \citep{osmani2026smgi}. SMGI adds not only a structural interface $\theta$, but an induced behavioral semantics $T_\theta$ and a governance discipline for norm-respecting change. We show that RoSO can be embedded into SMGI as a typed semantic layer, so that service descriptions become dynamically governable rather than merely well formed. This yields a RoSO-to-SMGI adequacy theorem, identity-preserving reconfiguration criteria, and compositional conditions under which locally acceptable updates remain globally admissible. The resulting claim is not that SMGI replaces RoSO, but that it provides a formal account of what admissible runtime change requires once service semantics must survive revision.
\end{abstract}

\section{Introduction}
The Robotic Service Ontology (RoSO) is one of the rare public service-robotics ontology efforts in which semantic standardization, deployment constraints, and revision pressure are all simultaneously visible. The 2019 OMG request for proposals already framed the problem in terms of communication, interoperability, and service composition from semantically consistent building blocks \citep{rosoRfp2019}. The subsequent trail from RoSO 1.0 beta to 1.0 beta 2 and the current 1.1 beta shows that the standard had to evolve precisely where semantic structure meets deployment reality: ontology references, parameter vocabularies, examples, and concept boundaries required explicit repair or clarification \citep{roso10beta1,roso10beta2,rosoAbout11}. RoSO is therefore not a narrow annotation vocabulary. It is a service-robotics semantic standard whose practical significance depends on whether typed service descriptions can survive revision, recomposition, and redeployment while preserving the protected commitments that the standard is meant to stabilize.

That practical dependency is the scientific difficulty addressed here. In service robotics, semantic well-formedness is necessary, but it is not yet enough. Once a service may be rebound to a different robot, repaired after failure, recomposed under new constraints, or migrated across environments, the central question becomes structural rather than merely taxonomic: when does the change remain an admissible transformation of the same protected service rather than an uncontrolled drift into another one? RoSO makes that question unusually legible because its public materials already connect service semantics with deployment, composition, interaction, and evolving constraints \citep{roso10beta2,roso11beta1}.

This article argues that the Structural Model of General Intelligence (SMGI) provides the right formal level for that question \citep{osmani2026smgi}. SMGI is not introduced here as a competing ontology and not as a replacement for standards. Its role is more specific and more fundamental: to supply a typed structural interface $\theta$, an induced behavioral semantics $T_\theta$, and a governance discipline for admissible change once semantic objects, tasks, and service relations have already been specified. On that reading, RoSO supplies semantic discipline; SMGI supplies the structural and dynamic conditions under which disciplined change can remain norm-respecting under substitution, rebinding, recomposition, and regime-sensitive redeployment.

The broader significance of that claim becomes clearer once RoSO is read against nearby standards and ontology-enabled service-robotics systems. The next section develops that comparison in terms of a single recurring problem: how to preserve the validity and identity of a semantically typed service under substitution, recomposition, transport across semantic layers, evaluator-sensitive change, and memory-conditioned reuse. RoSO then returns not as a limiting case, but as the article's central and most publicly traceable demonstration case.

\paragraph{Contributions.}
This article makes five integrated contributions.
\begin{enumerate}[leftmargin=2em]
    \item It reformulates the formal needs implicit in RoSO's scope, module structure, conformance points, and public revision history as structural requirements on an adaptive robotic service system.
    \item It defines an explicit embedding of RoSO into the SMGI meta-model, showing how ontology imports induce a typed representation, ontology-admissible service graph space, evaluator family, environment family, and structured memory.
    \item It proves the main adequacy results of the bridge: every specification-level RoSO-conformant service application induces a well-typed SMGI object; ontology-grounded reconfiguration becomes dynamically admissible when closure, stability, capacity control, and evaluative invariance are satisfied; and SMGI strictly extends ontology-only conformance by deciding reconfiguration problems that logical consistency alone cannot settle.
    \item It derives structural consequences for service robotics, including compositional admissibility, service-identity preservation under controlled change, and memory-aware reuse of previously certified or previously failed configurations.
    \item It supports the framework with a clause-wise adequacy map, two case studies, and an implementation blueprint, while arguing that the public RoSO revision trail is itself evidence that service-ontology standardization already generates problems of transport, substitution, and controlled revision that call for a framework like SMGI.
\end{enumerate}
\section{Related work}

This section is intentionally selective. It does not survey robotics at large, and it does not treat international breadth as a decorative objective. Its organizing question is narrower and more demanding: once service semantics, task representations, component models, and ontology-enabled reasoning already exist, what do nearby standards and scientific systems already provide, and what do they still not decide, when structural change must remain admissible under substitution, recomposition, transport across semantic layers, evaluator-sensitive change, and memory-conditioned reuse? The missing layer is not semantics alone but governed admissible dynamics: a law for preserving service validity and protected commitments while typed structures change. SMGI is positioned in this section not as another service ontology, but as the first framework in the paper that explicitly combines a typed structural interface, an induced behavioral semantics, and governance obligations for admissible change.

\subsection{Normative and ontological scaffolding beyond RoSO}
A first pressure comes from standards that already stabilize parts of the problem without trying to solve all of it at once. ISO~8373 fixes the vocabulary of robotics \citep{iso8373_2021}. ISO~18646-2 turns service-robot navigation performance into an object of explicit testing \citep{iso18646_2_2024}. ISO~22166-201 formalizes a common information model for service-robot modules with interoperability and composability in view \citep{iso22166_201_2024}. Together, these documents stabilize terms, measurable behavior, and typed modular artifacts. What they do not yet decide is whether a revised configuration still counts as the same admissible service under typed constraints. SMGI becomes relevant only above that baseline.

The IEEE ontology line exerts a comparable pressure on the representational side. CORA made robotics knowledge itself a standardizable ontological object \citep{schlenoff2012ieee}. IEEE~1872.1 does the same for robot task representation, explicitly framing tasks in terms of terms, attributes, structures, properties, constraints, and relations, including communication across ontology hierarchy levels \citep{ieee18721_2024}. IEEE~1872.2 extends the same effort toward autonomy, while P1872.3 pushes further toward multi-robot, cloud-connected, and trust-sensitive settings \citep{aur2021,p18723}. The significance of this line is not only that tasks and autonomy descriptors are no longer treated as informal metadata, but that representational hierarchies and constraint-bearing task structures are already explicit enough to make transformation problems unavoidable. Its limit is equally clear: these standards formalize representational strata, but not yet a general criterion for when movement across those strata preserves service validity rather than merely semantic well-formedness.

OMG makes that missing level unusually concrete. RoIS framed service-robot systems through standardized interaction and API concerns \citep{roisFramework2011}. RoIS 2.0 beta 2 shows that this service-oriented line remains public and active \citep{rois20beta2}. RoSO 1.0 beta 2 then turns services, functions, interactions, and environments into a deployable ontology program with machine-readable artifacts \citep{roso10beta2}, while the 1.1 beta trail makes its revision pressure publicly visible \citep{rosoAbout11}. These are not merely adjacent standards. They already presuppose that service semantics must survive deployment. The open question is therefore no longer whether services can be typed, but how their evolution should be judged beyond ontology conformance.

A secondary pressure comes from cloud-service formulations such as ITU-T Y.3533, which make functional requirements for robotics-as-a-service explicit \citep{ituY3533}. For the present paper, however, the central normative spine remains ISO, IEEE, and OMG, because they most directly structure vocabularies, task and service semantics, modular artifacts, and public ontology revision.

\subsection{Semantic interoperability and ontology-enabled robotic knowledge}
A second line of work shows that the problem is not exhausted by naming or standard taxonomy. The service-robot ontology review by Prestes et al. already argued that standards and ontologies matter because heterogeneous robotic artifacts must communicate and interoperate under shared semantic commitments \citep{prestes2013service}. Zager, Sieber, and Fay sharpen that point for autonomous mobile robots by proposing an information model for semantic interoperability rather than merely a richer terminology \citep{zager2024semanticinterop}. These works make intelligibility across platforms a first-class requirement. Their remaining limitation is that they characterize alignment and information sharing more directly than admissible transformation: they help explain when heterogeneous systems can talk about the same service world, but not yet when a changed configuration still preserves the protected service identity that matters operationally.

KnowRob pushes the same discussion from semantic alignment toward operational knowledge. It was explicitly designed to integrate encyclopedic knowledge, task descriptions, environment models, and observed actions in a common robotic knowledge-processing infrastructure \citep{tenorth2013knowrob}. Its ontology stack, including SOMA and alignment to foundational ontologies, makes activity structure, physical-social context, and situated behavior semantically tractable \citep{knowrobOntologies2025}. ORO pursues a related ambition for cognitive robot architectures by treating knowledge management itself as an operational robotics concern \citep{lemaignan2010oro}. These systems show that robotic semantics must already be contextual, executable, and structurally organized. What they do not yet provide is a general structural criterion for deciding when such organizations may be revised, transported, or recomposed while preserving service-level intelligibility and protected commitments. That second-order criterion is precisely where SMGI can add value.

The same diagnosis appears in work on dependable autonomy and focused semantic reasoning. Aguado et al. show that ontology-enabled processes are already used to support robustness, adaptation, resilience, and trust in robot autonomy \citep{aguado2024dependable}. Vitucci and Gini show that description-logical reasoning can directly influence grasp-related decision making in a constrained manipulation setting \citep{vitucci2019grasping}. These contributions matter because they demonstrate that ontologies in robotics are no longer merely classificatory; they already shape what systems can decide and execute. Their common boundary is architectural scope: they govern specific reasoning loops or dependable processes rather than admissible transformation across semantic layers.

This point is reinforced by two broader surveys that sit between foundational ontology work and fully runtime service adaptation. Carlucho et al. analyze semantic reasoning frameworks for robotic systems in terms of knowledge sources, world representations, tasks, actions, agents, and computational frameworks \citep{carlucho2022semanticreasoning}. Manzoor et al. survey ontology-based knowledge representation systems for autonomous robots across domestic, hospital, and industrial applications, repeatedly emphasizing flexibility, reuse, and adaptability as practical design pressures rather than purely taxonomic concerns \citep{manzoor2021survey}. These surveys matter because they show that the field already has many of the semantic ingredients needed for runtime intelligence. What remains comparatively weak is not semantic richness itself, but a general discipline for deciding when semantically rich robotic systems remain admissibly the same under revision, recomposition, and context change.

\subsection{Service-robot execution, task processing, and knowledge-rich adaptation}
A third line of work is even closer to RoSO because it already treats services, task restructuring, and execution under changing conditions rather than ontology as static description. In the Japanese service-ontology line, Kamei and Miyashita study standardization of ontology for service robots directly in relation to RoSO, while Ukai et~al. and Ngo et~al. earlier frame robot-technology ontologies around service provision and service-system construction \citep{kamei2021serviceology,ukai2009rtontology,ngo2011rtontology}. These works matter because they presuppose a service world in which ontology is already tied to provision, composition, and redeployment. In other words, they do not merely ask how to classify service entities; they already presuppose that service realizations will evolve. Their common limit is not lack of semantics but lack of a general criterion for when service-level evolution preserves the same protected service rather than merely producing another well-formed description.

The same pressure appears in ontology-enabled execution. Cui et~al. study semantic task planning for service robots in open worlds; Ge et~al. couple ontology, perception, planning, and execution in an autonomous task-processing framework for dynamic, unstructured environments \citep{cui2021semantic,ge2024ontology}. In both cases ontology is not only representational but operationally constitutive of runtime task restructuring. These systems therefore make dynamic admissibility unavoidable: service meaning must survive replanning, failure handling, and context change. What they still leave open is a general law for deciding when that restructuring remains identity-preserving rather than pipeline-local, and when local adaptation remains globally admissible across the service stack.

A further line enriches the semantic substrates on which service-level adaptation depends. Han et~al. survey semantic maps as carriers of environment-level meaning for mobile service robots, Miao et~al. propose a multi-layer knowledge-graph representation for robot manipulation, and Sun et~al. study ontology-based high-level decision making in search and rescue \citep{han2021semanticmaps,miao2023kgmanip,sun2019smart}. These works increase representational precision and decision relevance, but they also sharpen the unresolved question: when map updates, carrier-level revisions, or knowledge-graph-mediated substitutions occur, what guarantees that service-level commitments remain intact? Here too the missing layer is not richer representation alone, but governed admissible dynamics across those representational carriers. SMGI is strongest exactly at that point: it does not replace these substrates, but governs when their revision remains within a protected service identity rather than silently changing the service itself.

\subsection{What remains weakly unified, and why SMGI matters}
Across these standards and scientific systems, the recurring difficulties are not merely definitional. They concern semantic stabilization, task and service representation, interoperability across heterogeneous artifacts, controlled composition, revision under updated constraints, transport across semantic layers, memory-conditioned reuse, and preservation of service identity under regime-sensitive change. Existing work addresses those dimensions separately, locally, or at the level of representational scaffolding. What remains comparatively weakly unified is not semantics in the abstract, but a general account of admissible structural change: a law for preserving service validity and protected commitments while typed structures evolve.

SMGI is best positioned as a response to that missing layer \citep{osmani2026smgi}. It is not a competitor to RoSO, CORA, ISO terminology, or ontology-enabled service systems, and not a universal substitute for lower-level assurance methods. Its contribution is more specific and more structural: it couples a typed structural interface $\theta$ with an induced behavioral semantics $T_\theta$, then governs change through closure, stability, bounded capacity, and evaluative invariance. In that sense, nearby standards and systems already provide rich structure; SMGI adds the dynamic governance required when that structure must survive revision.

This is why RoSO matters so much for the present article. Among nearby standards, it is the public case in which service semantics, deployment constraints, ontology imports, parameter meaning, example repair, and versioned clarification already make the problem of admissible change unusually legible. The next section therefore returns to RoSO not as the only target of the framework, but as the case in which the argument can be made most explicitly.

\section{RoSO: semantic standardization with an implicit dynamic challenge}
\subsection{What RoSO standardizes}
The official OMG materials present RoSO as a standard for common vocabularies and ontologies for service robots. The public specification pages and normative documents emphasize that service robots operate in unpredictable, real environments; that higher-order services may need to be composed from smaller functional elements; and that interoperability requires a common descriptive layer for services, components, interactions, and their surrounding conditions \citep{roso10beta2,rosoAbout11}. The About-RoSO pages further state that RoSO defines vocabularies for functions and constraints of robotic functional components for deployment in robotic services, and that it is intended to work with IEEE 1872/CORA and other ontologies relevant to HRI/HAI and service domains \citep{rosoAbout11,roso10beta2}. 

In the 1.1 specification, RoSO is organized into ontology modules, including a core robotic service ontology, an interaction ontology, and a function ontology. The specification states that the vocabularies are classified under categories such as \emph{Agents, Services, Functions, and Environments} \citep{roso11beta1}. This classification is already extremely suggestive from the viewpoint of SMGI, because these are precisely the kinds of typed distinctions that a dynamic structural theory must internalize.

\subsection{What the 2025 revisions make explicit}
The recent public revision history is itself informative. The RoSO 1.0 FTF issue record shows that several nontrivial questions had to be resolved or deferred at the standard level: references to external ontologies and standards had to be updated; relationships to other ontologies such as Commons and CORA had to be described more explicitly; RoIS parameters had to be represented by names and relations rather than by bare classes; examples had to be rewritten in Turtle for interpretability; and the conceptual boundary around entities such as \enquote{Avatar} and \enquote{AvatarRobot} had to be repaired \citep{rosoIssues10}. The RoSO 1.1 issue record continues this trajectory by adding further terms for component parameters and by making event/action relations more explicit through notions such as \enquote{notifies} and \enquote{executes} \citep{rosoIssues11}. These are not superficial editorial adjustments. They reveal recurring structural tensions around transport across ontologies, semantic commensurability, parameter grounding, and concept revision. Our claim in this article is not that RoSO already contains a full dynamic theory, but that its public revision trail makes the need for such a theory unusually visible.

That point must be stated carefully. RoSO and its neighboring OMG artifacts clearly involve deployment, interaction, and sequences of component use \citep{roso10beta2,roisFramework2011}. What they do not yet provide is a general semantics of admissible structural change comparable to the coupled dynamics $(	heta,T_	heta)$ of SMGI. Read through the present framework, cross-ontology synchronization, parameter attachment, action/event clarification, and concept-boundary repair become public traces of transport, substitution, and controlled revision problems that call for governance stronger than ontology consistency alone.

\subsection{A concrete 2025 RoSO--SMGI mapping}
The purpose of the following mapping is not to redescribe the RoSO specification, but to make explicit how publicly visible 2025 standard elements instantiate the structural distinctions required by the present framework. The table is interpretive rather than claimed as part of the RoSO standard itself. It uses only elements that are explicit in the published 2025 trail: the normative ontology artifacts and high-level semantic categories documented in the released specifications \citep{roso10beta2,roso11beta1,rosoAbout11}, together with issue-driven clarifications around parameter semantics, ontology linkage, action/event relations, and conceptual repair \citep{rosoIssues10,rosoIssues11}.

\begin{table}[H]
\centering
\small
\begin{tabularx}{\textwidth}{>{\raggedright\arraybackslash}p{0.24\textwidth}>{\raggedright\arraybackslash}p{0.22\textwidth}>{\raggedright\arraybackslash}X>{\raggedright\arraybackslash}p{0.18\textwidth}}
\toprule
\textbf{Public RoSO element (2025)} & \textbf{Standard-level role} & \textbf{SMGI reading} & \textbf{Concrete gain from SMGI}\\
\midrule
RoboticServiceOntology, RoboticServiceFunctionOntology, RoboticServiceInteractionOntology \citep{rosoAbout11,roso10beta2} & Normative ontology modules for services, functions, and interactions & Typed representational interface $r_{\RoSO}$ and admissible service-hypothesis space $\Hh_{\RoSO}$ & Runtime change stays tied to ontology-grounded objects rather than untyped substitutions\\
\addlinespace
Agents, Services, Functions, Environments categories \citep{roso11beta1} & High-level semantic partition of the service domain & Minimal typed decomposition of regimes, evaluators, service roles, and environment-indexed constraints & Reconfiguration can be checked against regime-specific obligations instead of only global consistency\\
\addlinespace
RoIS parameter clarification and named parameter relations \citep{rosoIssues10,rosoIssues11} & Move from bare classes toward explicit parameter attachment and naming & Typed evaluator and state constraints; parameter grounding enters $r_{\RoSO}$ and $\Lcal_{\RoSO}$ & Substitution decisions can account for parameter-sensitive validity, not just class membership\\
\addlinespace
Explicit action/event relations such as \enquote{executes} and \enquote{notifies} \citep{rosoIssues11} & Clarification of operational relations in the interaction layer & Interface contracts inside service graphs and orchestration policies $\sigma$ & Compositional admissibility can be checked at subservice boundaries\\
\addlinespace
Updated links to Commons, CORA, and related standards \citep{rosoIssues10,roso10beta2} & Synchronization with adjacent ontologies and standards & Transport and commensurability constraints across imported semantic spaces & Cross-ontology updates are treated as controlled transport problems rather than opaque imports\\
\addlinespace
Avatar / AvatarRobot repair \citep{rosoIssues10,roso10beta2} & Concept-boundary clarification in the public standardization record & Controlled concept revision under an invariant semantic core & Distinguishes legitimate conceptual repair from uncontrolled vocabulary drift\\
\bottomrule
\end{tabularx}
\caption{Concrete mapping from public 2025 RoSO elements to their SMGI reading.}
\label{tab:roso-smgi-map}
\end{table}

\subsection{The dynamic challenge already present in the standard}
RoSO is not aimed only at static description. The specification contrasts static configuration, where developers manually compose suitable components in advance, with dynamic configuration, where service systems allocate and combine available components depending on service conditions and environmental constraints \citep{roso11beta1}. This dynamic orientation matters. A service ontology becomes practically valuable only when it can support not just description, but reconfiguration of real systems under changing environments, users, resources, and risks.

Yet this dynamic ambition opens a gap. Once a system may switch components, alter service graphs, add supervision loops, or change interaction protocols, semantic consistency is not sufficient. One also needs a law telling us when such change is valid. This is the gap the present article addresses.

\subsection{RoSO conformance and what it does not decide}
RoSO 1.1 distinguishes a specification-level conformance point and a linked-data-level conformance point. At specification level, relevant OWL ontologies are formally imported and must remain logically consistent. At linked-data level, one or more RoSO ontologies may be referenced without full import. In both cases, the standard requires correct use or mapping of OMG URIs for ontology elements \citep{roso11beta1}. This is a robust semantic discipline.

However, ontology conformance does not answer questions such as the following: if two different service compositions are both ontology-consistent, which one should a robot use under bandwidth degradation? If a robot is replaced at runtime, when does the service remain the \emph{same service} in a meaningful sense? If a previous configuration failed in a similar environment, how should that memory constrain the next composition choice? These are questions of dynamic admissibility, not import consistency.

\section{Why SMGI is the right formal level}
\subsection{The SMGI viewpoint}
SMGI was introduced as a formal theory of structural generalization in which the unit of analysis is not merely a policy or hypothesis inside a fixed environment, but the typed learning interface itself \citep{osmani2026smgi}. A system is represented by
\[
\theta=(r,\Hh,\Pi,\Lcal,\E,\Mcal),
\]
and its operational behavior by a realization map $T_{\theta}$. This matters here because service robotics does not change only predictions or control gains. It changes role bindings, service graphs, active constraints, evaluator priorities, and the reuse status of prior certified or failed configurations. SMGI is therefore relevant not only as a structural language but as a language of induced behavioral semantics and governed change.

\subsection{Why service robotics needs a structural admissibility theory}
A service robot may need to substitute one perception module for another, migrate execution from one robot to another, insert a human oversight step, change route planning under congestion, or tighten safety constraints after sensing degradation. These are not mere parameter updates. They are transformations of the service interface itself. RoSO gives the vocabulary for stating such transformations. SMGI gives the formal conditions under which they remain acceptable. The complementarity is therefore direct: standards define what must hold, while SMGI governs how change may occur while preserving what must hold.

\subsection{The four governance obligations specialized to RoSO}
For the purposes of this article, the four SMGI obligations are best read as four conditions of norm-respecting service governance. Together they answer four questions that standards alone leave only partially determined: does the transformation stay inside the authorized structural space, does it preserve a service-level stability envelope, does it remain governable rather than combinatorially unconstrained, and does it preserve a protected service core across evaluator shifts?

\paragraph{Closure.} A reconfiguration must remain inside the space of typed service transformations authorized by the imported ontologies and active service graph. Closure is what prevents ontology-grounded change from degenerating into arbitrary rewiring.

\paragraph{Stability.} A sequence of admissible updates must preserve a service-level stability certificate, such as bounded switching cost, bounded degradation of protected outputs, or preservation of a certified safety envelope. Stability is what prevents silent semantic or operational drift.

\paragraph{Bounded capacity.} Reconfiguration cannot be allowed to explode into an ungovernable search over arbitrarily rich service graphs. Ontology restriction helps, but SMGI makes complexity control explicit through priors, budgets, and monitorable structural limits.

\paragraph{Evaluative invariance.} Service robotics is regime-sensitive: emergency delivery, routine delivery, and customer guidance do not weigh latency, supervision, and cost in the same way. SMGI therefore allows evaluator families, but requires that a protected service core remain invariant across regime shifts. Evaluative invariance is what makes adaptive change compatible with enduring commitments.

\subsection{What SMGI concretely adds to RoSO}
For RoSO, the gain from SMGI is not generic sophistication but governed admissibility. The framework turns ontology-grounded runtime change into an explicit decision problem about substitution, rebinding, revision, memory-aware reuse, and identity-preserving recomposition. It therefore adds neither another service ontology nor a competing standard, but the structural and dynamic discipline under which a semantically typed service may continue to count as the same protected service while it changes. This is also why $T_\theta$ matters: RoSO already fixes typed semantic commitments, whereas SMGI adds the induced behavioral semantics through which those commitments are tested under actual reconfiguration.

\section{Embedding RoSO into the SMGI meta-model}
\subsection{Construction principle}
Let $\mathcal O_{\RoSO}$ denote a set of RoSO ontologies imported at specification-level conformance. Let $\mathcal C$ be the set of available robotic components, $\mathcal S$ the set of service tasks, and $\mathcal U$ the set of user/deployment contexts. Let $K$ be the knowledge base formed by the imported ontologies, application-specific assertions, and runtime service facts.

\begin{definition}[RoSO-induced SMGI meta-model]
Given $(\mathcal O_{\RoSO},K,\mathcal C,\mathcal S,\mathcal U)$, define
\[
\theta_{\RoSO}=(r_{\RoSO},\Hh_{\RoSO},\Pi_{\RoSO},\Lcal_{\RoSO},\E_{\RoSO},\Mcal_{\RoSO})
\]
as follows:
\begin{enumerate}[leftmargin=2em]
    \item $r_{\RoSO}$ maps raw service, robot, interaction, and environment states to typed RoSO-compatible semantic states;
    \item $\Hh_{\RoSO}$ is the space of ontology-admissible service-composition hypotheses (service graphs, assignments, orchestration policies, and active constraints);
    \item $\Pi_{\RoSO}$ is a structural prior over compositions, favoring for instance simpler, previously certified, or domain-preferred structures;
    \item $\Lcal_{\RoSO}$ is a family of regime-indexed evaluators scoring task performance, safety, semantic feasibility, switching cost, and reuse value;
    \item $\E_{\RoSO}$ is a family of deployment regimes induced by ontology-grounded environment descriptions, user profiles, and platform conditions;
    \item $\Mcal_{\RoSO}$ stores previously validated service graphs, failure signatures, substitution rules, and context-qualified certificates.
\end{enumerate}
\end{definition}

This is not an arbitrary encoding. It mirrors the ontology modules themselves. Service and function vocabularies shape $\Hh$; interaction vocabularies constrain orchestration and evaluators; environment vocabularies shape $\E$; and deployment history becomes structured memory.

\subsection{Semantic state, hypotheses, and realization}
Given a raw runtime state $x_t$, the system builds a semantic state
\[
z_t=r_{\RoSO}(x_t)\in \Z,
\]
where $\Z$ is the set of RoSO-grounded semantic states. A hypothesis $h\in\Hh_{\RoSO}$ is a candidate service realization. We write
\[
h=(G,\alpha,\sigma,\kappa),
\]
where $G\in\G$ is a service graph, $\alpha$ assigns components to roles/functions, $\sigma$ specifies an orchestration policy over actions, events, and interactions, and $\kappa$ encodes active constraints.

The operational behavior is realized by $T_{\theta_{\RoSO}}$, which maps $(z_t,\Mcal_{\RoSO})$ to a candidate update or execution decision. SMGI does not require that the realization be implemented by any specific AI paradigm. It may be symbolic, hybrid, learning-based, or mixed, provided the admissibility obligations are respected.

\subsection{First adequacy result}
\begin{proposition}[Well-typed embedding of RoSO into SMGI]\label{prop:embeddingv3}
Every specification-level RoSO-conformant service application induces a well-typed SMGI meta-model $\theta_{\RoSO}$.
\end{proposition}

\begin{proof}
Specification-level conformance guarantees a logically consistent ontology import closure and URI discipline. The imported ontology graph therefore determines a coherent typed vocabulary for services, functions, interactions, and environments. This makes $r_{\RoSO}$ well-defined. The same ontology constrains admissible service graphs, thereby inducing a well-defined $\Hh_{\RoSO}$. Service requirements and regime descriptors induce $\Lcal_{\RoSO}$ and $\E_{\RoSO}$. Historical executions and certificates define $\Mcal_{\RoSO}$. Since each component of the tuple is typed by ontology-grounded semantics, the induced SMGI object is well formed.
\end{proof}

\section{From ontology conformance to dynamic admissibility}
\subsection{RoSO-governed transformations}
\begin{definition}[RoSO-governed transformation]
A RoSO-governed transformation is a map $\tau\in\Tcal$ acting on service configurations such that both the pre- and post-states admit representation through $r_{\RoSO}$ and the structural differences are expressible as ontology-grounded substitutions, additions, removals, rebinding of service roles, or updates of active constraints.
\end{definition}

Examples include replacing one perception component by another, migrating a service to another robot, adding a human-supervision subservice, or restricting a route because a zone becomes unavailable.

\subsection{Admissibility conditions for RoSO-governed transformations}
\paragraph{Closure.} At the theorem level, a RoSO transformation must preserve type soundness. The transformed service graph must remain expressible in the imported ontologies, and all service-function-interaction-environment dependencies must remain satisfiable.

\paragraph{Stability.} At the theorem level, the transformed service realization must preserve a service-level stability certificate, such as bounded switching cost, bounded degradation in service quality, or preservation of a certified safety envelope.

\paragraph{Capacity control.} At the theorem level, dynamic service composition must not explode into an unconstrained search. Ontology restrictions reduce this space, and SMGI further controls it through structural priors and explicit complexity budgets.

\paragraph{Evaluative invariance.} At the theorem level, different service regimes may require different evaluators, but a protected service core must remain invariant. For instance, a medication-delivery service may allow different latency trade-offs in emergency and routine modes while still requiring traceability and non-violation of hard safety constraints.

\subsection{Main adequacy theorem}
\begin{theorem}[RoSO-to-SMGI dynamic adequacy]\label{thm:mainadequacyv3}
Let $(\theta_{\RoSO},T_{\theta_{\RoSO}})$ be the RoSO-induced SMGI system. Suppose that for every RoSO-governed transformation $\tau\in\Tcal$ the following hold:
\begin{enumerate}[leftmargin=2em,label=(A\arabic*)]
    \item the transformed service graph remains type-sound with respect to the imported ontologies;
    \item there exists a certificate $V$ such that the service-level drift induced by $\tau$ is bounded by $V$;
    \item the induced update preserves a bounded complexity budget relative to $\Pi_{\RoSO}$;
    \item evaluator changes preserve an invariant service core $\Phi$ consisting of semantic task identity and mandatory safety constraints.
\end{enumerate}
Then every such transformation is SMGI-admissible. Consequently, RoSO-guided dynamic service reconfiguration can be treated as a certified structural evolution problem under SMGI.
\end{theorem}

\begin{proof}
Condition (A1) is structural closure. Condition (A2) provides the required stability certificate. Condition (A3) ensures bounded structural capacity. Condition (A4) guarantees evaluative invariance. Since these are precisely the four SMGI obligations, every transformation satisfying (A1)--(A4) is admissible in the SMGI sense.
\end{proof}

\section{Strengthening the bridge: identity preservation and composition}
\subsection{Service identity preservation under reconfiguration}
A recurrent practical objection to dynamic service composition is that successful reconfiguration may still silently change the service into a different service. To address this, we make the identity question explicit.

\begin{definition}[Service identity functional]
Let $\Id(h,z)$ be a functional measuring preservation of service identity for service realization $h$ under semantic state $z$. The functional aggregates: (i) preservation of the service request class, (ii) preservation of mandatory service outputs, (iii) preservation of hard safety obligations, and (iv) preservation of essential interaction obligations.
\end{definition}

\begin{definition}[Identity-preserving transformation]
A RoSO-governed transformation $\tau$ is identity-preserving if for all relevant semantic states $z$,
\[
\Id(\tau(h),z) \ge \eta,
\]
for a certified threshold $\eta$ determined by the service domain.
\end{definition}

\begin{proposition}[Identity preservation from invariant core]\label{prop:identity}
If the invariant evaluative core $\Phi$ includes the full service identity functional $\Id$, then every SMGI-admissible RoSO-guided transformation is identity-preserving.
\end{proposition}

\begin{proof}
By admissibility, evaluator updates preserve the invariant core $\Phi$. If $\Id\subseteq \Phi$, then no admissible transformation may violate the threshold required by $\Id$. Therefore every admissible transformation preserves service identity.
\end{proof}

Proposition~\ref{prop:identity} matters because it turns a vague engineering requirement (``the service should remain the same service'') into a falsifiable admissibility constraint. In the present paper, the intended reading is simple: service identity is preserved iff the protected service-level commitments encoded in the invariant evaluator core remain above the certified threshold under the admissible transformation class.

\subsection{Compositional admissibility}
Service robots rarely execute monolithic services. They execute compositions of subservices such as navigation, handoff, interaction, verification, and exception handling. We therefore strengthen the adequacy bridge with a compositional result.

\begin{definition}[Interface-compatible subservice family]
A family of subservices $\{h_i\}_{i=1}^m$ is interface-compatible if every pair of adjacent subservices agrees on the ontology-grounded types of exchanged entities, events, and obligations, and if the induced composed graph remains in $\Hh_{\RoSO}$.
\end{definition}

\begin{theorem}[Compositional admissibility]\label{thm:composition}
Let $h = h_1 \circ \cdots \circ h_m$ be a composed RoSO-grounded service graph from an interface-compatible subservice family. Suppose each local update $h_i \mapsto h_i'$ is SMGI-admissible and preserves the interface contracts with adjacent subservices. Then the global transformation
\[
h_1 \circ \cdots \circ h_m \mapsto h_1' \circ \cdots \circ h_m'
\]
is SMGI-admissible.
\end{theorem}

\begin{proof}
By local admissibility, each update preserves closure, stability, capacity budget, and evaluative invariance relative to its local obligations. Interface compatibility ensures that composition of the transformed subservices remains type-sound and that no hidden incompatibility is introduced at subservice boundaries. Since the global graph remains in $\Hh_{\RoSO}$, the composed update preserves closure. The global stability budget is bounded by the sum of local budgets plus bounded interface terms. Capacity remains bounded by subadditivity of the chosen complexity budget. The invariant core is preserved because interface contracts require consistent propagation of mandatory service obligations. Therefore the composed transformation is admissible.
\end{proof}

This theorem is important for RoSO because the standard itself is meant to support coherent composition of higher-order services from consistent building blocks. Read operationally, Theorem~\ref{thm:composition} is the formal bridge from ontology-grounded local validity to higher-order service assembly: it states when locally acceptable substitutions, rebinding operations, or constraint updates remain globally admissible once they interact through shared service interfaces. The compositional theorem therefore preserves the spirit of RoSO service composition at the dynamic-certification level rather than only at the level of ontology-consistent design.

\section{Why the extension is nontrivial}
\subsection{SMGI strictly extends ontology-only conformance}
\begin{proposition}[Strict extension over ontology conformance]\label{prop:strictv3}
There exist service-reconfiguration questions that cannot be settled by ontology-level conformance alone but are decidable within the RoSO-induced SMGI formulation.
\end{proposition}

\begin{proof}
Ontology-level conformance checks import consistency and URI-correct use of ontology elements. Consider a service platform facing network degradation and choosing between two component substitutions, both of which preserve ontology consistency. One choice has lower switching cost, preserves a latency budget in the active regime, and matches memory of prior successful deployments in analogous environments; the other does not. These criteria are not reducible to OWL consistency. They depend on evaluator structure, capacity penalties, and memory reuse, encoded respectively in $\Lcal_{\RoSO}$, $\Pi_{\RoSO}$, and $\Mcal_{\RoSO}$. Hence the choice is decidable in the SMGI formulation but not through ontology conformance alone.
\end{proof}

\subsection{Memory as an explicit structural operator}
RoSO permits extension and reuse but does not define a formal memory operator. In service robotics this matters because service validity is path-dependent. A platform that has repeatedly observed that a speech interface fails in a noisy retail environment should not treat each new service-composition decision as memoryless. In SMGI, memory is explicit: $\Mcal$ stores context-conditioned admissibility information, previous certificates, and failure signatures. This turns ontology use from static semantic annotation into history-sensitive structural governance.

\subsection{Evaluator families and regime-sensitive service identity}
RoSO standardizes services, functions, and related constraints, but it does not independently formalize evaluator families. In practice, service robotics is regime-sensitive. A hospital-delivery service and a retail-guidance service may share component types while differing sharply in acceptable latency, handoff protocol, human-supervision policy, and cost structure. SMGI makes such differences explicit by modeling $\Lcal=\{\ell_e: e\in\E\}$. Crucially, this does not fragment service identity because SMGI requires a regime-invariant evaluative core.

\section{Formal and operational consequences for RoSO}
This section makes explicit what the RoSO--SMGI bridge yields once the main theorems are in place. The point is no longer to reargue the bridge, but to spell out the formal and operational consequences it imposes on a service platform.

\subsection{Semantic state}
Let $x_t$ denote the raw platform state at time $t$, including robot availability, sensor streams, user requests, network status, safety flags, and environmental facts. The representation map produces a semantic state
\[
z_t = r_{\RoSO}(x_t) \in \Z,
\]
where $\Z$ is the set of RoSO-grounded semantic state descriptions. Typical components of $z_t$ include the service request class and parameters, currently available agents and components, component-function relations, active interaction state, environment descriptors, and regime-specific constraints such as deadlines, user category, safety class, bandwidth, and location restrictions.

\subsection{Hypothesis space and service graphs}
A hypothesis $h\in\Hh_{\RoSO}$ is a candidate service realization. Let us write
\[
h=(G,\alpha,\sigma,\kappa),
\]
where $G\in\G$ is a service graph, $\alpha$ assigns components to functional/service roles, $\sigma$ specifies an orchestration policy over actions, events, and interactions, and $\kappa$ specifies active constraints. Only ontology-grounded assignments are admitted. This already gives a key formal advantage: $\Hh_{\RoSO}$ is typically much smaller and more structured than an unconstrained architecture-search space.

\subsection{Evaluator family}
For each regime $e\in\E_{\RoSO}$ define an evaluator
\[
\ell_e(h,z_t)
= w_e^{\mathrm{task}}J_{\mathrm{task}}(h,z_t)
+ w_e^{\mathrm{safety}}J_{\mathrm{safety}}(h,z_t)
+ w_e^{\mathrm{semantic}}J_{\mathrm{semantic}}(h,z_t)
+ w_e^{\mathrm{cost}}J_{\mathrm{cost}}(h,z_t)
+ w_e^{\mathrm{reuse}}J_{\mathrm{reuse}}(h,z_t).
\]
The service-invariant evaluative core is represented by a functional
\[
\Phi(h,z_t) = J_{\mathrm{identity}}(h,z_t)+J_{\mathrm{hard\_safety}}(h,z_t),
\]
which must remain above a threshold or whose violation immediately renders the transformation inadmissible.

\subsection{Memory and structural reuse}
Let memory store tuples of the form
\[
(e,h,\text{certificate},\text{outcome},\text{failure signature},\text{reuse tag}).
\]
This supports positive reuse of previously certified service structures, negative reuse through penalization of failure-associated structures, and certificate transport from prior graphs to nearby graphs under controlled conditions.

\subsection{Proof obligations for a RoSO-compliant service platform}
\subsection{Typed substitution admissibility}
\begin{definition}[Typed substitution admissibility]
A substitution $c_1\leadsto c_2$ is admissible if:
\begin{enumerate}[leftmargin=2em]
    \item $c_2$ is mapped to the required RoSO function class or to a certified refinement thereof,
    \item all dependent service roles remain satisfiable,
    \item all evaluator-core constraints in $\Phi$ remain certified,
    \item the induced transition cost remains within the regime budget,
    \item reuse of past certificates does not contradict current environment constraints.
\end{enumerate}
\end{definition}

\begin{lemma}[Closure under certified substitution]
If the above conditions hold, then $c_1\leadsto c_2$ preserves structural closure of $(\theta_{\RoSO},T_{\theta_{\RoSO}})$.
\end{lemma}

\begin{proof}
Function typing and role satisfiability ensure that the transformed graph remains inside $\Hh_{\RoSO}$. Preservation of evaluator-core constraints and transition budget ensures that the transformation remains inside the admissible transformation family. Hence closure is preserved.
\end{proof}

\subsection{Bounded regime-switch drift}
\begin{assumption}[Bounded regime-switch cost]
For each switch $e\to e'$ there exists a nonnegative cost $C(e,e')$ and a redeployment map $\Gamma_{e\to e'}$ such that the realized service degradation satisfies
\[
\Delta_{e\to e'} \le C(e,e') + \varepsilon_{e\to e'},
\]
where $\varepsilon_{e\to e'}$ is a bounded adaptation residual.
\end{assumption}

\begin{proposition}[Regime-switch stability]
Under bounded regime-switch cost, the RoSO-induced SMGI system admits a service-level stability certificate of bounded drift across evaluator changes.
\end{proposition}

\begin{proof}
Let $V_t$ be cumulative deviation from the invariant core $\Phi$. Each switch contributes at most $C(e,e')+\varepsilon_{e\to e'}$. Summing over a finite sequence of switches yields bounded drift. Therefore a service-level stability certificate exists.
\end{proof}

\subsection{Capacity reduction induced by ontology restriction}
\begin{proposition}[Ontology-induced capacity reduction]\label{prop:capacityv3}
Let $\Hh_{\mathrm{raw}}$ be the unconstrained set of service graphs and $\Hh_{\RoSO}\subseteq \Hh_{\mathrm{raw}}$ the subset admissible under the RoSO ontology. Then for any complexity measure $\mathcal C$ monotone under set inclusion,
\[
\mathcal C(\Hh_{\RoSO}) \le \mathcal C(\Hh_{\mathrm{raw}}).
\]
Therefore ontology-grounded typing strengthens the feasibility of SMGI capacity control.
\end{proposition}

\begin{proof}
By definition, $\Hh_{\RoSO}$ excludes service graphs violating type and relation constraints imposed by the ontology; therefore it is a subset of $\Hh_{\mathrm{raw}}$. Monotonicity of $\mathcal C$ under inclusion yields the claim.
\end{proof}

\section{Case study I: hospital delivery under dynamic component allocation}
Consider an indoor hospital-delivery service. The service transports medicine or samples from pharmacy to ward under varying traffic, battery state, elevator access, network quality, and human-interaction constraints.

\subsection{RoSO description of the scenario}
In RoSO terms, the scenario includes:
\begin{itemize}[leftmargin=2em]
    \item \textbf{services:} delivery request, pickup, navigation, handoff, confirmation, exception handling,
    \item \textbf{agents:} robots, nurses, pharmacist, patient, remote supervisor,
    \item \textbf{functions:} localization, path planning, obstacle detection, identity checking, human notification, item handoff,
    \item \textbf{interaction structure:} actions, events, conditions, action results, escalation events,
    \item \textbf{environments:} corridor, pharmacy, ward, restricted area, low-connectivity segment, elevator bottleneck.
\end{itemize}

\subsection{Why static conformance is insufficient}
Suppose robot $R_1$ begins the task but its battery becomes critical near a low-connectivity ward. A second robot $R_2$ is available on another floor. The ontology can describe both robots and their functions. But it cannot, by itself, settle whether handoff from $R_1$ to $R_2$ preserves service identity, whether an intermediate human confirmation is now required, or whether the resulting configuration remains inside the acceptable safety envelope.

\subsection{SMGI formulation of the decision}
The semantic state $z_t$ includes robot availability, battery, floor access, user priority, congestion, connectivity, medication traceability status, and current service deadline. Candidate hypotheses include:
\begin{enumerate}[leftmargin=2em]
    \item continue with $R_1$ and degrade speed,
    \item reroute $R_1$ through charging access and delay the task,
    \item transfer to $R_2$ with nurse-mediated handoff,
    \item escalate to supervised execution.
\end{enumerate}
Each hypothesis is RoSO-grounded. The evaluator family differs by regime: emergency mode penalizes delay sharply, whereas routine mode emphasizes robustness and disturbance minimization. The invariant core requires that medication not be lost, misdelivered, or delivered in violation of safety or traceability requirements.

\subsection{Certified handoff transformation}
Suppose $R_1\leadsto R_2$ is proposed. The admissibility test checks:
\begin{enumerate}[leftmargin=2em]
    \item function equivalence or certified refinement for navigation and handoff roles,
    \item compatibility of $R_2$ with the destination ward route,
    \item preservation of traceability and notification obligations,
    \item bounded transition cost,
    \item support from memory of prior successful handoffs in analogous contexts.
\end{enumerate}
If these pass, the handoff is admissible. If not, the platform may fall back to supervised execution. The point is not that the ontology disappears during runtime; rather, it remains the semantic backbone of the admissibility test.

\section{Case study II: retail guidance with interaction-sensitive regime switching}
To stress that the argument is not limited to logistical delivery, consider a retail-guidance service deployed in a shopping environment. The service must greet a customer, infer or retrieve the desired destination, plan a route, guide the customer through dynamic pedestrian flow, and possibly hand off to a remote human assistant if speech recognition or localization becomes unreliable.

\subsection{RoSO description of the scenario}
The relevant RoSO-grounded entities include:
\begin{itemize}[leftmargin=2em]
    \item \textbf{services:} greeting, intent acquisition, route guidance, promotional recommendation, human escalation,
    \item \textbf{agents:} guide robot, customer, staff member, remote operator,
    \item \textbf{functions:} speech interaction, localization, route planning, display interaction, escalation signaling,
    \item \textbf{environments:} entrance, aisle, dense pedestrian zone, noisy zone, restricted zone.
\end{itemize}

\subsection{Dynamic challenge}
Suppose the active speech interface becomes unreliable in a noisy zone while crowd density increases. Two ontology-consistent strategies are available: (i) switch to touch-screen + visual guidance while maintaining robot-led navigation; or (ii) initiate remote human assistance and convert the robot to a display/escort role. Both are semantically describable. The question is which preserves service identity and remains admissible under current risk, cost, and user-state constraints.

\subsection{Why SMGI matters here}
In this scenario, evaluator switching is central: user-comfort, latency, commercial relevance, and handoff burden trade off differently across quiet and noisy regimes. The invariant core may require that the service remain a valid guidance service, that the user not be abandoned, and that routing safety constraints remain satisfied. Memory matters because past failures of speech interaction in similar zones should affect the next reconfiguration choice. This is exactly the kind of regime-sensitive, history-sensitive problem that ontology conformance alone cannot settle but SMGI can formalize.

\section{Runtime and assurance implications}
The formal bridge also suggests a compact runtime architecture. This section and the next three do not reopen the theoretical case; they consolidate its operational implications through orchestration, clause-wise alignment, falsification pressure, and runtime contract design. A platform first computes an ontology-grounded semantic state $z_t=r_{\RoSO}(x_t)$ from perception, network, and service-process facts. It then enumerates candidate service graphs $h\in\Hh_{\RoSO}$ subject to RoSO typing and relation constraints. A regime detector selects the active evaluator in $\Lcal_{\RoSO}$ and screens candidates through the four SMGI obligations together with service-identity checks. If several candidates survive, $\Mcal_{\RoSO}$ is used to privilege previously certified structures and penalize known failure motifs. Deployment and monitoring then update memory with certificates or failure signatures. In short, the architecture contains five auditable stages: semantic lift, candidate generation, admissibility screening, memory-aware ranking, and certified deployment. The point of this blueprint is not implementation detail for its own sake, but to show that the RoSO--SMGI bridge induces a runtime contract rather than only an interpretive reading.

\subsection{Clause-wise adequacy: from RoSO requirements to SMGI obligations}
This section makes the bridge more explicit by aligning the semantic and operational needs expressed by RoSO with the corresponding formal duties imposed by SMGI. The purpose is not to claim that the standard itself states the SMGI framework, but to show that the standard's own motivations point toward exactly the sort of structural governance SMGI provides.

\begin{table}[H]
\centering
\small
\begin{tabularx}{\textwidth}{>{\raggedright\arraybackslash}p{0.23\textwidth}>{\raggedright\arraybackslash}p{0.28\textwidth}>{\raggedright\arraybackslash}X}
\toprule
\textbf{RoSO-level need} & \textbf{Operational interpretation} & \textbf{SMGI response} \\
\midrule
Common vocabulary for service robots, functions, and environments & Semantic interoperability across heterogeneous service components and deployments & Typed representation map $r_{\RoSO}$ and ontology-grounded state space $\mathcal Z$ constrain what may be represented and compared \\
Consistent composition of higher-order services from well-defined building blocks & Service graphs must be built from semantically compatible subservices & Hypothesis space $\mathcal H_{\RoSO}$ and compositional admissibility theorem constrain assembly and update of service graphs \\
Dynamic configuration and allocation according to situation & Runtime substitution, rebinding, and redeployment of service components & Transformation family $\mathcal T$ plus four admissibility obligations determine when change is valid \\
Environment-sensitive deployment and user constraints & Service validity depends on context, not only on nominal capability & Environment family $\mathcal E_{\RoSO}$ and regime-indexed evaluators $\mathcal L_{\RoSO}$ make context-dependence explicit \\
Logical consistency and URI-correct conformance & Semantic discipline for imported ontologies and references & Structural closure starts from ontology conformance but extends it to runtime update validity \\
Extensibility to user/domain-specific needs & Services may be extended or specialized without destroying semantic meaning & Memory-aware reuse and evaluator-preserving structural updates govern controlled extension \\
Need for interoperable yet dependable runtime behavior & Semantic validity should survive deployment stress, switching, and partial degradation & Stability certificates, bounded complexity budgets, and invariant service core supply assurance beyond static conformance \\
\bottomrule
\end{tabularx}
\caption{Clause-wise adequacy map: how RoSO-level requirements are absorbed and strengthened by SMGI.}
\label{tab:clausewise}
\end{table}

Table~\ref{tab:clausewise} makes explicit what is otherwise implicit in the preceding proofs. RoSO gives the system a semantics of components, services, functions, interactions, and environments; SMGI turns that semantics into a runtime-governed admissibility discipline. This is the central sense in which SMGI is an adequate solution: it preserves the normative role of the standard while adding the missing dynamic law.

\begin{proposition}[Adequacy by semantic absorption]
Suppose a service platform satisfies specification-level RoSO conformance and uses the imported ontologies to define its representation map, composition space, regime family, and memory indexing. If the platform enforces the SMGI obligations over every ontology-grounded transformation, then every operational decision remains semantically anchored in the RoSO vocabulary while also being dynamically certifiable.
\end{proposition}

\begin{proof}
Semantic anchoring follows from the use of the imported ontologies inside $r_{\RoSO}$, $\mathcal H_{\RoSO}$, and $\mathcal E_{\RoSO}$. Dynamic certifiability follows from closure, stability, capacity control, and evaluative invariance being enforced for each admissible transformation. Thus the platform is simultaneously RoSO-grounded and SMGI-governed.
\end{proof}

\subsection{Falsification and evaluation protocol}
A theory of adequacy should not remain purely declarative. The RoSO--SMGI bridge is therefore falsifiable on structural rather than merely task-success metrics. The critical benchmark families are substitution, regime-switch, environment-shift, and memory-reuse scenarios. The critical measurements are identity-preservation rate, safe reconfiguration rate, bounded degradation under regime change, certificate-reuse gain, and structural regret against weaker baselines. The right baselines are (i) a purely ontology-conformant orchestrator with no admissibility layer, (ii) a typed planner with no memory-aware certification, and (iii) a heuristic runtime manager with memory but no evaluator-family formalization. The strong form of the thesis is falsified if such weaker stacks systematically match the RoSO--SMGI system on those structural metrics under repeated substitution, regime change, and recurrent failure. This protocol matters methodologically because it aligns empirical evaluation with the paper's real claim: not just better pointwise performance, but a better law of valid structural change.

\subsection{Failure modes and boundary conditions}
A rigorous adequacy argument should also specify where the bridge can fail. We therefore state the principal failure modes explicitly.

\paragraph{Failure of semantic closure.}
If a proposed reconfiguration introduces components or role bindings that cannot be typed in the imported RoSO ontologies, then the transformed service graph exits $\Hh_{\RoSO}$ and the transformation is inadmissible. This is not a weakness of SMGI but a correctly detected failure of standard-grounded service meaning.

\paragraph{Failure of evaluator-core preservation.}
A transformation may remain semantically well formed while violating the invariant service core $\Phi$. For example, a substitution may preserve all ontology references and still break medication traceability or mandatory human-notification obligations. Such a transformation is RoSO-conformant at the ontology level but inadmissible in the stronger SMGI sense.

\paragraph{Failure of bounded structural complexity.}
A platform that admits unconstrained online addition of service-graph motifs may lose any realistic hope of certification or selection. SMGI therefore exposes a genuine engineering requirement: ontological typing must be supplemented with priors, budgets, or monitorable restriction policies.

\paragraph{Failure of memory governance.}
If prior service failures are stored but never made operational in the transition law, then memory exists only descriptively, not structurally. The RoSO--SMGI bridge requires not merely logging but certified reuse or certified quarantine.

\begin{proposition}[Boundary-condition proposition]\label{prop:boundaryv5}
A RoSO-conformant service platform fails to realize the full adequacy thesis of this paper if and only if at least one of the following holds: (i) transformations are not closure-preserving, (ii) no bounded-drift certificate is available, (iii) no effective structural-complexity control exists, or (iv) evaluator updates fail to preserve the invariant service core.
\end{proposition}

\begin{proof}
The forward direction is immediate from Theorem~\ref{thm:mainadequacyv3}: violating any one of the four obligations blocks admissibility. For the reverse direction, if all four conditions hold, then the adequacy theorem applies and the full thesis is realized. Hence the failure set is exactly the complement of the admissible set characterized by the theorem.
\end{proof}

This proposition is useful because it converts a broad philosophical claim into a finite checklist of failure modes. A standards body, integrator, or platform auditor can therefore ask not only whether RoSO is imported correctly, but whether the four adequacy blockers have been eliminated.

\subsection{Algorithmic contract for a RoSO--SMGI orchestrator}
The architectural ideas above can be compressed into an algorithmic contract.

\paragraph{Input.}
At decision time $t$, the orchestrator receives a raw platform state $x_t$, ontology imports $\mathcal O_{\RoSO}$, a current service graph $h_t$, a regime label $e_t$, and memory state $m_t\in\Mcal_{\RoSO}$.

\paragraph{Step 1: semantic lift.}
Compute $z_t=r_{\RoSO}(x_t)$. Reject the update immediately if the candidate transformed graph cannot be typed under the ontology imports.

\paragraph{Step 2: candidate generation.}
Enumerate or sample a restricted family of ontology-grounded candidate transformations $\tau_1,\dots,\tau_k$ from the admissible transformation grammar.

\paragraph{Step 3: certificate screening.}
For each candidate, test: closure, regime-specific stability bounds, complexity budget, and preservation of the invariant core $\Phi$.

\paragraph{Step 4: memory-aware ranking.}
Among candidates that pass screening, compute evaluator scores using the active evaluator $\ell_{e_t}$ plus memory terms rewarding certified reuse and penalizing motifs associated with prior failures.

\paragraph{Step 5: certified deployment.}
Deploy the highest-ranked admissible candidate and store the resulting service execution, certificate updates, and failure/success traces back into $\Mcal_{\RoSO}$.

\begin{remark}[Why the contract matters]
The contract shows that the proposal is not only interpretive. It induces an auditable runtime law. A platform claiming to implement the RoSO--SMGI bridge should be able to point to explicit semantic lift, admissible candidate generation, certificate screening, memory-aware ranking, and certified deployment stages.
\end{remark}
\section{Limitations and non-claims}
The adequacy claim of this article is substantial but not unlimited.

First, we do not claim that RoSO itself proves dynamic correctness. The standard provides the semantic scaffolding needed to formulate dynamic correctness, but not the full admissibility theory.

Second, we do not claim that every robotics platform using RoSO automatically satisfies SMGI obligations. The obligations must be instantiated with concrete certificates, monitors, and transition laws. Our theorems are conditional theorems of adequacy, not automatic correctness results for all possible implementations.

Third, we do not claim that SMGI replaces low-level control theory or domain-specific safety assurance. In healthcare, mobility, retail, or collaborative robotics, local safety constraints, verification procedures, and regulation remain indispensable. SMGI adds a cross-regime structural governance layer; it does not subsume every lower-level assurance technique.

Fourth, the present article reasons primarily at the kernel level of SMGI. Richer memory stratification, empirical topology-shift benchmarks, and more detailed evaluator-update theory would strengthen future applied developments.

\section{Discussion}
The core result of the paper can now be stated more sharply. RoSO makes the semantic and normative commitments of a robotic service explicit; SMGI supplies the structural and dynamic conditions under which those commitments may survive change. The bridge matters for four specific reasons.

\paragraph{RoSO becomes operationally actionable.}
A semantic standard has maximal impact when it guides not only data description but also system construction and runtime governance. The RoSO--SMGI bridge turns ontology use into a path toward adaptive service orchestration.

\paragraph{Service adaptation becomes falsifiable.}
Without a structural framework, claims that a reconfiguration is valid remain informal. Under SMGI, they reduce to explicit obligations: closure, stability, bounded capacity, evaluative invariance, and identity preservation. Semantic change becomes something that can be certified, rejected, or monitored rather than merely narrated.

\paragraph{Semantics is upgraded into assurance.}
Many systems can annotate components semantically. Far fewer can justify adaptive reconfiguration under semantic, service, and safety constraints. The RoSO--SMGI combination upgrades ontology use from annotation to assurance because it couples typed structure with an induced behavioral semantics and a governance discipline for norm-respecting change. In that precise sense, norms specify what must hold, while SMGI governs how change may occur without violating what must hold.

\paragraph{Composition becomes certifiable.}
Because robotic services are often assembled from subservices, the compositional theorem is not a minor technicality. It is the formal bridge from ontology-grounded local service validity to higher-order service assembly: it states when locally acceptable substitutions remain globally admissible at the service level.

Public RoSO standardization materials reinforce the point. Cross-ontology references had to be synchronized with evolving standards, relations to external ontologies had to be made explicit, component parameters had to be semantically attached rather than left as bare classes, examples had to be reformulated for interpretability, and even core concepts had to be redrawn to preserve clean category boundaries \citep{rosoIssues10,rosoIssues11}. The framework proposed here therefore does more than add another formal vocabulary. It explains why such tensions recur and why some of them are better understood as problems of transport, commensurability, controlled revision, and memory-aware reconfiguration rather than as isolated specification edits. The RoSO treatment of \enquote{Avatar} and \enquote{AvatarRobot} is emblematic: what looks superficially like a naming issue is in fact a repair of conceptual boundary conditions \citep{rosoIssues10,roso10beta2}.

It is useful to distinguish four levels of outcome. Some tensions are \emph{clarified}: ontology consistency is shown to be weaker than admissible runtime change. Some are \emph{structurally reframed}: cross-ontology links, parameter attachment, and action/event relations are reinterpreted as transport and interface-governance problems. Some are \emph{partially resolved}: service-identity preservation and higher-order service composition receive explicit formal criteria under stated assumptions. Others remain \emph{open}: the practical synthesis of certificates, monitors, and regime-sensitive evaluators for large service ecosystems is still an implementation problem.

\section{Conclusion}
RoSO provides what service robotics urgently needs at the semantic level: formal vocabularies for robotic services, interaction, function, and environment-grounded deployment constraints \citep{rosoPortal,roso11beta1,rosoAbout11}. It also makes clear that service robotics increasingly operates under dynamic configuration and dynamic allocation of heterogeneous components. What RoSO does not itself provide is a general law of valid structural change.

This article has argued that SMGI is an adequate formal answer to that missing layer. We showed that a specification-level RoSO-conformant service application induces a well-typed SMGI meta-model; that ontology-grounded service realizations can be embedded as hypotheses in a typed service-graph space; and that the four SMGI governance obligations specialize naturally to robotic-service reconfiguration. On that basis we derived a dynamic adequacy theorem, a strict extension over ontology-only conformance, an identity-preservation result, and a compositional admissibility theorem.

The resulting picture is simple. RoSO defines protected semantic commitments. SMGI governs the admissible coupled dynamics through which those commitments may survive rebinding, recomposition, revision, and redeployment. The framework therefore does not replace the standard; it explains what respecting the standard requires once change becomes unavoidable. The next step is implementation: an ontology-grounded service orchestrator with explicit SMGI certifiers for reconfiguration, regime switching, and memory-aware reuse. Such a system would provide a direct empirical test of the adequacy thesis advanced here.

\section*{Acknowledgments}
Gratitude is due to the broader OMG RoSO standardization community whose collective work, from the 2019 request-for-proposals stage through the 2025 beta revisions, sharpened the semantic and structural questions studied in this article \citep{rosoRfp2019,roso10beta2,rosoAbout11}. Public specification materials identify submitting, contributing, and supporting organizations across this trajectory, including Japan Robot Association (JARA), Knowledge Applications and Research (KAR), the National Institute of Advanced Industrial Science and Technology (AIST), Shibaura Institute of Technology, Universit\'e Sorbonne Paris Nord, ATR, ETRI, the University of Seoul, UPEC, NTT, and the Object Management Group itself \citep{roso10beta2,rosoAbout11}. The present article also benefited from direct exposure to this standardization effort through the LIPN / Universit\'e Sorbonne Paris Nord side, which helped sharpen the practical sensitivity of the paper to interoperability, ontology revision, and service-level semantic transport, while leaving the scientific claims to stand on their own.

\bibliographystyle{plainnat}
\bibliography{SMGI_RoSO_article_v14}

@misc{osmani2026smgi,
  author = {Aomar Osmani},
  title = {A New Formal Theory of AGI Part I: The Structural Admissibility Kernel (SMGI)},
  year = {2026},
  note = {arXiv preprint}
}

@misc{rosoRfp2019,
  author = {{Object Management Group}},
  title = {Robotic Service Ontology Request for Proposal},
  year = {2019},
  howpublished = {OMG document robotics/2019-06-01},
  url = {https://www.omg.org/cgi-bin/doc?robotics/2019-06-01},
  note = {Accessed 2026-03-31}
}

@misc{rosoPortal,
author = {Suyoung Chi and Abdelghani Chibani and Koji Kamei and Elisa Kendall and Aomar Osmani and Takashi Yoshimi},
 title = {Robotic Service Ontology (RoSO), version 1.0 formal, Robotics Domain Task Force (DTF)},
  year = {2026},
  howpublished = {Official OMG robotics portal},
  url = {https://www.omg.org/spec/ROSO/1.0/About-ROSO},
  note = {Accessed 2026-03-31}
}

@misc{rosoAbout11,
  author = {Suyoung Chi and Abdelghani Chibani and Koji Kamei and Elisa Kendall and Aomar Osmani and Takashi Yoshimi},
  title = {About the Robotic Service Ontology Specification Version 1.1 beta},
  year = {2025},
  howpublished = {Official OMG specification page},
  url = {https://www.omg.org/spec/ROSO/1.1/Beta1/PDF},
  note = {Accessed 2026-03-31}
}

@misc{roso11beta1,
  author = {Suyoung Chi and Abdelghani Chibani and Koji Kamei and Elisa Kendall and Aomar Osmani and Takashi Yoshimi},
  title = {Robotic Service Ontology (RoSO) Version 1.1 Beta 1},
  year = {2025},
  howpublished = {Official OMG specification},
  url = {https://www.omg.org/spec/ROSO/1.1/Beta1/PDF},
  note = {Accessed 2026-03-31}
}

@misc{roso10beta2,
  author = {Suyoung Chi and Abdelghani Chibani and Koji Kamei and Elisa Kendall and Aomar Osmani and Takashi Yoshimi},
  title = {Robotic Service Ontology (RoSO) Version 1.0 Beta 2},
  year = {2025},
  howpublished = {Official OMG specification},
  url = {https://www.omg.org/spec/ROSO/1.0/Beta2/PDF},
  note = {Accessed 2026-03-31}
}

@misc{roso10beta1,
  author = {Suyoung Chi and Abdelghani Chibani and Koji Kamei and Elisa Kendall and Aomar Osmani and Takashi Yoshimi},
  title = {Robotic Service Ontology (RoSO) Version 1.0 Beta 1},
  year = {2023},
  howpublished = {Official OMG specification},
  url = {https://www.omg.org/spec/ROSO/1.0/Beta1/PDF},
  note = {Accessed 2026-03-31}
}

@misc{rosoIssues10,
  author = {Suyoung Chi and Abdelghani Chibani and Koji Kamei and Elisa Kendall and Aomar Osmani and Takashi Yoshimi},
  title = {Issues associated with the Robotic Service Ontology Specification Version 1.0 beta},
  year = {2023},
  howpublished = {Official OMG issue listing},
  url = {https://www.omg.org/spec/ROSO/1.0/Beta1/About-ROSO},
  note = {Accessed 2026-03-31}
}

@misc{rosoIssues11,
  author = {Suyoung Chi and Abdelghani Chibani and Koji Kamei and Elisa Kendall and Aomar Osmani and Takashi Yoshimi},
  title = {Issues associated with the Robotic Service Ontology Specification Version 1.1 beta},
  year = {2025},
  howpublished = {Official OMG issue listing},
  url = {https://www.omg.org/spec/ROSO/1.1/Beta1/About-ROSO},
  note = {Accessed 2026-03-31}
}

@misc{rois20beta2,
  author = {Suyoung Chi and Abdelghani Chibani and Koji Kamei and Elisa Kendall and Aomar Osmani and Takashi Yoshimi},
  title = {About the Robotic Interaction Service Framework Specification Version 2.0 beta 2},
  year = {2025},
  howpublished = {Official OMG specification page},
  url = {https://www.omg.org/spec/RoIS/2.0/Beta2/About-RoIS},
  note = {Accessed 2026-03-31}
}

@inproceedings{roisFramework2011,
  author = {Kamei, Koji and Nishio, Shuichi and Hagita, Norihiro and Sato, Miki},
  title = {RoIS Framework: Standardization of API for Service Robot Systems},
  booktitle = {Proceedings of the 44th Hawaii International Conference on System Sciences (HICSS)},
  year = {2011},
  pages = {1--8},
  doi = {10.1109/HICSS.2011.280}
}

@article{schlenoff2012ieee,
  author = {Schlenoff, Craig and Balakirsky, Stephen and Kramer, Thomas and Prestes, Edson and Goncalves, Paulo and Li, Howard},
  title = {An {IEEE} Standard Ontology for Robotics and Automation},
  journal = {IEEE Intelligent Systems},
  year = {2012},
  volume = {27},
  number = {2},
  pages = {99--101},
  doi = {10.1109/MIS.2012.20}
}

@misc{ieee18721_2024,
  author = {{IEEE}},
  title = {{IEEE} 1872.1-2024: {IEEE} Standard for Robot Task Representation},
  year = {2024},
  howpublished = {IEEE Standards Association standard page},
  url = {https://standards.ieee.org/ieee/1872.1/6993},
  note = {Accessed 2026-03-31}
}

@misc{aur2021,
  author = {{IEEE}},
  title = {{IEEE} 1872.2-2021: {IEEE} Standard for Autonomous Robotics (AuR) Ontology},
  year = {2021},
  howpublished = {IEEE Standards Association standard page},
  url = {https://standards.ieee.org/ieee/1872.2/7094},
  note = {Accessed 2026-03-31}
}

@misc{p18723,
  author = {{IEEE}},
  title = {{IEEE} P1872.3: Standard for Ontology Reasoning on Multiple Robots},
  year = {2022},
  howpublished = {IEEE Standards Association project page},
  url = {https://standards.ieee.org/ieee/1872.3/11037},
  note = {Accessed 2026-03-31}
}

@misc{iso8373_2021,
  author = {{International Organization for Standardization}},
  title = {ISO 8373:2021 Robotics --- Vocabulary},
  year = {2021},
  howpublished = {ISO standard page},
  url = {https://www.iso.org/standard/75539.html},
  note = {Accessed 2026-03-31}
}

@misc{iso22166_201_2024,
  author = {{International Organization for Standardization}},
  title = {ISO 22166-201:2024 Robotics --- Modularity for service robots --- Part 201: Common information model for modules},
  year = {2024},
  howpublished = {ISO standard page},
  url = {https://www.iso.org/standard/82334.html},
  note = {Accessed 2026-03-31}
}

@misc{ituY3533,
  author = {{ITU-T}},
  title = {Recommendation ITU-T Y.3533: Cloud computing --- Functional requirements for robotics as a service},
  year = {2023},
  howpublished = {ITU recommendation database},
  url = {https://www.itu.int/ITU-T/recommendations/rec.aspx?rec=15746},
  note = {Approved 2023-12-14; accessed 2026-03-31}
}

@article{prestes2013service,
  author = {Prestes, Edson and Carbonera, J. Luiz and Fiorini, Sandro and Jorge, Vanderlei A. M. and Abel, Mara and Madhavan, Raj and Locoro, Angelo and Goncalves, Paulo and Barreto, Maur{\'\i}cio E. and Habib, Maki and Chibani, Amine and Gerard, S{\'e}bastien and Amirat, Yacine and Schlenoff, Craig},
  title = {Applied Ontologies and Standards for Service Robots},
  journal = {Robotics and Autonomous Systems},
  year = {2013},
  volume = {61},
  number = {11},
  pages = {1215--1227},
  doi = {10.1016/j.robot.2013.04.008}
}

@article{tenorth2013knowrob,
  author = {Tenorth, Moritz and Beetz, Michael},
  title = {KnowRob --- A Knowledge Processing Infrastructure for Cognition-Enabled Robots},
  journal = {The International Journal of Robotics Research},
  year = {2013},
  volume = {32},
  number = {5},
  pages = {566--590},
  doi = {10.1177/0278364913481635}
}

@misc{knowrobOntologies2025,
  author = {{KnowRob Project}},
  title = {KnowRob Ontologies and {SOMA}},
  year = {2025},
  howpublished = {KnowRob project documentation},
  url = {https://www.knowrob.org/ontologies},
  note = {Accessed 2026-03-31}
}

@inproceedings{lemaignan2010oro,
  author = {Lemaignan, S{\'e}verin and Ros, Raquel and M{\"o}senlechner, Lorenz and Alami, Rachid and Beetz, Michael},
  title = {{ORO}, a Knowledge Management Platform for Cognitive Architectures in Robotics},
  booktitle = {2010 IEEE/RSJ International Conference on Intelligent Robots and Systems (IROS)},
  year = {2010},
  pages = {3548--3553},
  doi = {10.1109/IROS.2010.5649547}
}

@article{zager2024semanticinterop,
  author = {Zager, Marvin and Sieber, Christoph and Fay, Alexander},
  title = {Towards Semantic Interoperability: An Information Model for Autonomous Mobile Robots},
  journal = {Journal of Intelligent \& Robotic Systems},
  year = {2024},
  volume = {110},
  pages = {123},
  doi = {10.1007/s10846-024-02159-3}
}

@article{cui2021semantic,
  author = {Cui, Guowei and Shuai, Wei and Chen, Xiaoping},
  title = {Semantic Task Planning for Service Robots in Open Worlds},
  journal = {Future Internet},
  year = {2021},
  volume = {13},
  number = {2},
  pages = {49},
  doi = {10.3390/fi13020049}
}

@article{ge2024ontology,
  author = {Ge, Yueguang and Zhang, Shaolin and Cai, Yinghao and Lu, Tao and Wang, Haitao and Hui, Xiaolong and Wang, Shuo},
  title = {Ontology Based Autonomous Robot Task Processing Framework},
  journal = {Frontiers in Neurorobotics},
  year = {2024},
  volume = {18},
  pages = {1401075},
  doi = {10.3389/fnbot.2024.1401075}
}

@article{sun2019smart,
  author = {Sun, Xiaolei and Zhang, Yu and Chen, Jing},
  title = {High-Level Smart Decision Making of a Robot Based on Ontology in a Search and Rescue Scenario},
  journal = {Future Internet},
  year = {2019},
  volume = {11},
  number = {11},
  pages = {230},
  doi = {10.3390/fi11110230}
}

@misc{iso18646_2_2024,
  author = {{International Organization for Standardization}},
  title = {ISO 18646-2:2024 Robotics --- Performance criteria and related test methods for service robots --- Part 2: Navigation},
  year = {2024},
  howpublished = {ISO standard page},
  url = {https://www.iso.org/standard/82643.html},
  note = {Accessed 2026-03-31}
}

@article{aguado2024dependable,
  author = {Aguado, Esther and Gomez, Virgilio and Hernando, Miguel and Rossi, Claudio and Sanz, Ricardo},
  title = {A Survey of Ontology-Enabled Processes for Dependable Robot Autonomy},
  journal = {Frontiers in Robotics and AI},
  year = {2024},
  volume = {11},
  pages = {1377897},
  doi = {10.3389/frobt.2024.1377897}
}

@article{vitucci2019grasping,
  author = {Vitucci, Nicola and Gini, Giuseppina},
  title = {Reasoning on Objects and Grasping Using Description Logics},
  journal = {Advanced Robotics},
  year = {2019},
  volume = {33},
  number = {13},
  pages = {616--635},
  doi = {10.1080/01691864.2019.1638452}
}

@article{ukai2009rtontology,
  author = {Ukai, Ken and Ando, Yoshinobu and Mizukawa, Makoto},
  title = {Robot Technology Ontology Targeting Robot Technology Services in Kukanchi --- Interactive Human-Space Design and Intelligence},
  journal = {Journal of Robotics and Mechatronics},
  year = {2009},
  volume = {21},
  number = {4},
  pages = {489--497},
  doi = {10.20965/jrm.2009.p0489}
}

@article{ngo2011rtontology,
  author = {Ngo, Trung Lam and Lee, Haeyeon and Mizukawa, Makoto},
  title = {Automatic Building Robot Technology Ontology Based on Basic-Level Knowledge},
  journal = {Journal of Robotics and Mechatronics},
  year = {2011},
  volume = {23},
  number = {4},
  pages = {515--522},
  doi = {10.20965/jrm.2011.p0515}
}

@article{han2021semanticmaps,
  author = {Han, Xiaoning and Li, Shuailong and Wang, Xiaohui and Zhou, Weijia},
  title = {Semantic Mapping for Mobile Robots in Indoor Scenes: A Survey},
  journal = {Information},
  year = {2021},
  volume = {12},
  number = {2},
  pages = {92},
  doi = {10.3390/info12020092}
}

@article{miao2023kgmanip,
  author = {Miao, Runqing and Wang, Yong and Wang, Chen and Zhang, Zexu and Li, Jinyang and Li, Yue},
  title = {Semantic Representation of Robot Manipulation with Knowledge Graph},
  journal = {Entropy},
  year = {2023},
  volume = {25},
  number = {4},
  pages = {657},
  doi = {10.3390/e25040657}
}

@article{kamei2021serviceology,
  author = {Kamei, Koji and Miyashita, Takahiro},
  title = {Standardization of Ontology for Service Robots},
  journal = {Serviceology},
  year = {2021},
  volume = {7},
  number = {4},
  pages = {138--141},
  doi = {10.24464/serviceology.7.4_138}
}

@article{carlucho2022semanticreasoning,
  author = {Carlucho, Ignacio and De Paula, Matias and Acosta, Guillermo Gabriel and Simoes, Evandro and Ghallab, Malik},
  title = {A survey of Semantic Reasoning frameworks for robotic systems},
  journal = {Robotics and Autonomous Systems},
  year = {2022},
  volume = {154},
  pages = {104294},
  doi = {10.1016/j.robot.2022.104294}
}

@article{manzoor2021survey,
  author = {Manzoor, Sumaira and Rocha, Yuri Goncalves and Joo, Sung-Hyeon and Bae, Sang-Hyeon and Kim, Eun-Jin and Joo, Kyeong-Jin and Kuc, Tae-Yong},
  title = {Ontology-Based Knowledge Representation in Robotic Systems: A Survey Oriented toward Applications},
  journal = {Applied Sciences},
  year = {2021},
  volume = {11},
  number = {10},
  pages = {4324},
  doi = {10.3390/app11104324}
}

\end{document}